\documentclass[letterpaper, 10 pt, conference]{ieeeconf}  

\IEEEoverridecommandlockouts                              

\usepackage{mathptmx} 
\usepackage{amsmath} 
\usepackage{amssymb}  
\usepackage{graphicx}
\usepackage{cite}
\usepackage{siunitx}

\newcommand{\bmat}[1]{\begin{bmatrix}#1\end{bmatrix}}
\newcommand{\tp}{\mathsf{T}}
\newcommand{\norm}[1]{\lVert{#1}\rVert}
\newcommand{\R}{\mathbb{R}}
\newtheorem{theorem}{Theorem}

\newtheorem{remark}{Remark}

\title{\LARGE \bf
Improved Scalable Lipschitz Bounds for Deep Neural Networks
}

\author{
Usman Syed$^{1}$ and Bin Hu$^{1}$
\thanks{*This work was supported by  the
AFOSR award FA9550-23-1-0732 and the NSF award
CAREER-2048168.}
\thanks{$^{1}$Usman Syed and Bin Hu are with the Coordinated Science Laboratory
and the Department of Electrical and Computer Engineering, University
of Illinois Urbana–Champaign, Champaign, IL 61801 USA
        {\tt\small usyed3@illinois.edu; binhu7@illinois.edu}}
}

\begin{document}

\maketitle
\thispagestyle{empty}
\pagestyle{empty}

\begin{abstract}
Computing tight Lipschitz bounds for deep neural networks is crucial for analyzing their robustness and stability, but existing approaches either produce relatively conservative estimates or rely on semidefinite programming (SDP) formulations (namely the LipSDP condition) that face scalability issues.
Building upon ECLipsE-Fast, the state-of-the-art Lipschitz bound method that avoids SDP formulations, we derive a new family of improved scalable Lipschitz bounds that can be combined to outperform ECLipsE-Fast. Specifically, we leverage more general parameterizations of feasible points of LipSDP to derive various closed-form Lipschitz bounds, avoiding the use of SDP solvers. In addition, we show that our technique encompasses ECLipsE-Fast as a special case and leads to a much larger class of scalable Lipschitz bounds for deep neural networks. Our empirical study shows that our bounds improve ECLipsE-Fast, further advancing the scalability and precision of Lipschitz estimation for large neural~networks.

\end{abstract}

\section{INTRODUCTION}

Neural networks (NNs) have achieved remarkable success across various domains such as computer vision, natural language processing, and feedback control.  The (global) Lipschitz constant that measures the maximum possible ratio of output change to input perturbation in the $\ell_2$-to-$\ell_2$ sense, is a key metric for neural network robustness and generalization~\cite{tsuzuku2018lipschitz, bartlett2017spectrally}. Importantly, for a neural network controller, its Lipschitz constant can also be used to  certify the closed-loop stability~\cite{jin2020stability,fazlyab2021introduction,fabiani2022reliably}. 
However, computing the exact Lipschitz constant for deep neural networks is NP-hard \cite{LipNP2018}, leading to significant research efforts in developing computationally tractable Lipschitz upper bounds.

Practical approaches for computing (global) $\ell_2$ Lipschitz bounds of NNs mainly fall into two categories. The first category relies on semidefinite programming (SDP) formulations. Specifically, for $[0,1]$-slope-restricted activation functions, the LipSDP approach \cite{fazlyab2019efficient} provides
the least conservative Lipschitz upper bounds that can be computed with polynomial time guarantees, paving the way for many developments in analysis and design of robust NNs \cite{pauli2021training,pauli2022neural,revay2020lipschitz,araujo2023a,wang2022,havens2023exploiting,pauli2023novel,barbara2024robust,wang2023direct}. However, how to scale LipSDP for very large NNs is still an open research question being heavily investigated \cite{xue2022chordal,pauli2023lipschitz,wang2024scalability,pauli2024lipschitz,xueclipse}.

In contrast, the second category of bounding methods relies on closed-form Lipschitz bounds that are computationally more scalable than SDP-based methods but often produce conservative estimates. 
For example, the so-called matrix norm product bound \cite{szegedy2014intriguing} gives arguably the most popular Lipschitz bound for large-scale NNs, and can be efficiently computed using power iteration or other advanced methods \cite{horn2012matrix,delattre2023efficient,grishina2024tight}. However,
this bound is based on the simple fact that the global Lipschitz constant of a NN can be upper bounded by the
product of Lipschitz constants of individual layers, and can be quite conservative in many situations.
More recently, ECLipsE-Fast \cite{xueclipse} has emerged as the state-of-the-art method in this second category, 
leveraging a recursive analytical decomposition approach to achieve less conservative bounds than the naive matrix norm product approach consistently while
preserving a similar level of scalability.\footnote{In \cite{xueclipse}, another SDP-based method termed as ECLipsE has also been developed. However, ECLipsE-Fast is much more scalable than ECLipsE.} 
However, there is still some gap between ECLipsE-Fast and the original LipSDP, and it is natural to ask whether one can derive other closed-form Lipschitz bounds for NNs to reduce the conservatism in ECLipsE-Fast.

In this paper, we provide an affirmative answer to the above question and present a new family of closed-form Lipschitz bounds that significantly advances the state of the art in scalable Lipschitz bound computation. Our approach extends the recursive framework of ECLipsE-Fast via bridging it with more general analytical parameterizations of the feasible points for LipSDP.
Specifically, built upon analytical solutions of a particular small matrix inequality, we 
derive a diverse set of new scalable Lipschitz bounds that can be combined to improve the original EClipsE-Fast method. This generalization allows us to maintain the computational efficiency of avoiding SDP solvers while substantially reducing the conservatism in ECLipsE-Fast. We show theoretically that our proposed bounds strictly encompass ECLipsE-Fast as a special case, and provide extensive numerical study to demonstrate that our bounds consistently improve upon ECLipsE-Fast across various network architectures.

\section{Preliminaries}
\label{sec:preliminiries}

\subsection{Problem Statement: Lipschitz Estimation of NNs}
Consider a standard feedforward NN that maps the input $x\in \R^{n_0}$ to the output $f_\theta(x)$ as follows:
\begin{align}\label{eq:NN}
\begin{split}
x^0&=x\\
x^k&=\phi(W_k x^{k-1}+b_k), \,k=1,\ldots, l\\
f_\theta(x)&=W_{l+1} x^l+b_{l+1},
\end{split}
\end{align}
where $W_k\in \R^{n_k\times n_{k-1}}$ and $b_k \in \R^{n_k}$ denote the weight matrix and the bias vector at the $k$-th layer,
respectively.  The network parameter is denoted by $\theta = \{W_k, b_k\}_{k=1}^{l+1}$.
The activation function $\phi$ is assumed to be $[0,1]$-slope-restricted. Standard activation functions including ReLU, \text{sigmoid}, and $\tanh$ all satisfy this assumption. 
The output $f_\theta(x)$ is an $n_{l+1}$-dimensional vector.  Denote the 2-norm of any vector as $\norm{\cdot}$, and the goal of the Lipschitz constant estimation is to   
 find a tight Lipschitz bound $L > 0$ such
that the following inequality holds
\begin{align}\label{eq:Lip}
\norm{f_\theta(x)-f_\theta(y)} \le L\norm{x-y},\,\forall x,y\in \R^{n_0}. 
\end{align}
Next, we will review several existing approaches for obtaining such Lipschitz bounds.

\subsection{Norm Product Bound}

The simplest Lipschitz upper bound for \eqref{eq:NN} is the so-called matrix norm product bound. Specifically, under the assumption that the activation $\phi$ is 1-Lipschitz, \eqref{eq:Lip} will hold with the following choice of $L$ \cite{szegedy2014intriguing}:
\begin{align}\label{eq:productbound}
L=\prod_{k=1}^{l+1} \sigma_{\max}(W_k)
\end{align}
where $\sigma_{\max}(\cdot)$ denotes the largest singular value (spectral norm). 
Essentially, the product of the spectral norm of every layer’s weight naturally leads to a Lipschitz bound.  The above product bound can be efficiently computed using the power iteration method\footnote{For convolutional layers, more advanced methods such as Gram iteration \cite{delattre2023efficient} and tensor norm bounds \cite{grishina2024tight} are available.}. The ease of use makes \eqref{eq:productbound} arguably the most popular Lipschitz bound for large-scale NNs at the practical level. However, the main issue is that \eqref{eq:productbound} can be overly conservative. Intuitively, the product bound holds for any 1-Lipschitz activation functions, and do not really exploit the $[0,1]$-slope-restricted property of ReLU, \text{sigmoid}, and $\tanh$ activations. This motivates extensive recent research efforts in developing improved Lipschitz bounds for NNs.

\subsection{LipSDP and Variants}

In  \cite{fazlyab2019efficient}, semidefinite programming (SDP) techniques have been leveraged to develop the LipSDP method, which provides the least conservative Lipschitz upper bounds with polynomial time guarantees. 
The LipSDP result \cite{fazlyab2019efficient,pauli2021training,pauli2022neural} states that the neural network \eqref{eq:NN} is $\sqrt{\gamma}$-Lipschitz if the following matrix inequality is feasible
\begin{align}\label{eq:LipSDP}
\bmat{
    I & -W_1^\tp  \Lambda_1  & 0 & \cdots & 0 \\
   - \Lambda_1 W_1 & 2\Lambda_1  &  \ddots & \ddots & \vdots\\
   0 & \ddots & \ddots & -W_l^\tp \Lambda_l & 0\\
   \vdots & \ddots &  -\Lambda_l W_l & 2\Lambda_l  & -W_{l+1}^\tp  \\
   0 & \cdots & 0 & -W_{l+1} &  \gamma I
   }\succeq 0,
\end{align}
where the decision variable matrix $\Lambda_k$ is required to be diagonal with non-negative entries for $k=1,\ldots, l$. 
For a given NN with $\{W_k\}_{k=1}^{l+1}$ fixed, the above condition is linear in $\{\Lambda_k\}_{k=1}^l$ and $\gamma$. To obtain tighter Lipschitz bounds, one can minimize $\gamma$ subject to the above linear matrix inequality (LMI) condition \eqref{eq:LipSDP}. This leads to a SDP problem which is termed as LipSDP. The formulation of the LMI condition \eqref{eq:LipSDP} requires exploiting the $[0,1]$-slope-restrictedness of the activation functions, significantly reducing the conservatism of the matrix norm product bound \eqref{eq:productbound}. 

However, for deep and wide neural networks used in practice, LipSDP faces severe scalability issues. 
There has been some recent progress in improving the scalability of LipSDP by either decomposing it into smaller SDPs \cite{xue2022chordal,pauli2024lipschitz,xueclipse} or transforming it to equivalent forms that can be solved using first-order iterative optimization algorithms \cite{wang2024scalability}. Despite such progress,  LipSDP-based methods are in general still much less scalable than the matrix norm product bound technique. More research efforts are still needed to mitigate the scalability/tightness trade-off in computing Lipschitz bounds for NNs.

\subsection{ECLipsE-Fast}

Interestingly, a recently-developed method termed as ECLipsE-Fast  \cite{xueclipse} provides a meaningful middle ground between the matrix norm product bound \eqref{eq:productbound} and LipSDP \eqref{eq:LipSDP}. 
ECLipsE-Fast is developed based on choosing a specific feasible point\footnote{In \cite{xueclipse}, some geometrical explanations are provided to justify why this specific feasible point is a reasonable choice for deriving Lipschitz bounds.} for LipSDP \eqref{eq:LipSDP} to derive a computable Lipschitz bound that does not require solving any SDPs. ECLipsE-Fast leverages a recursive scheme to compute the Lipschitz bound. Specifically, initialize $M_1=I_{n_0}$, and choose $ \Lambda_1=\frac{1}{\sigma_{\max}(W_1 M_1^{-1} W_1^\tp)}I_{n_1}$.  For $k=2, \cdots, l$, perform the following recursion:
\begin{align}\label{eq:ECLipsE}
M_k&=2\Lambda_{k-1}-\Lambda_{k-1} W_{k-1} M_{k-1}^{-1} W_{k-1}^\tp \Lambda_{k-1},\\\label{eq:ECLipsE2}
\Lambda_k& =\frac{1}{\sigma_{\max}(W_k M_k^{-1} W_k^\tp)} I
\end{align}
Then the NN \eqref{eq:NN} is guaranteed to be $\sqrt{\sigma_{\max}(W_{l+1}M_{l+1}^{-1}W_{l+1}^\tp)}$-Lipschitz with $M_{l+1}=2\Lambda_l-\Lambda_l W_l M_l^{-1} W_l^\tp \Lambda_l$. 

ECLipsE-Fast does not require solving SDPs, and the required number of matrix norm
calculations is on the order of $O(l)$. When applying the power iteration to compute $\sigma_{\max}(W_k M_k^{-1} W_k^\tp)$ for each $k$, one does not need to invert $M_k$ explicitly. Notice that the power iteration method only requires computing $W_k M_k^{-1} W_k^\tp v$ fast for any given $v$, which is a much simpler task than computing $M_k^{-1}$.

Empirically, it has been observed that  ECLipsE-Fast consistently gives much less conservative numerical results than the matrix norm product bound \eqref{eq:productbound}. In addition, ECLipsE-Fast is much more scalable than SDP-based Lipschitz analysis methods \cite{xue2022chordal,pauli2024lipschitz,xueclipse,wang2024scalability} despite being more conservative than these methods.
We note that the analysis in \cite{xueclipse} provides some geometrical explanations to justify the use of ECLipsE-Fast on the conceptual level. 
However, there is no formal theory addressing whether ECLipsE-Fast
is optimal in any sense.
It is very natural to ask whether we can derive other SDP-free Lipschitz bounds that are less conservative than ECLipsE-Fast. In this next section, we will answer this question by deriving new Lipschitz bounds that can be used to improve ECLipsE-Fast.

\section{Main Results}
\label{sec:main}

In this section, we present our main results. First, we show that we can characterize the feasible points of LipSDP via analyzing a small matrix inequality whose size does not scale with the NN depth. Next, we convert various analytical closed-form solutions of this matrix inequality to  a new family of scalable Lipschitz bounds for NNs. Finally, we discuss how to combine our derived bounds to reduce the conservatism in ECLipsE-Fast.

\subsection{Characterizing Feasible Points of LipSDP}

If one can express some feasible points of LipSDP \eqref{eq:LipSDP} as closed-form functions of $\{W_k\}_{k=1}^l$, then those expressions naturally lead to Lipschitz bounds for \eqref{eq:NN}. As a matter of facts, this perspective unifies the norm product bound \eqref{eq:productbound} and the ECLipsE-Fast, both of which can be derived from  characterizing special feasible points of \eqref{eq:LipSDP}. Specifically, it is easy to verify that choosing $\Lambda_k=\frac{1}{\prod_{m=1}^k \sigma_{\max}^2(W_m)}I$ and $\gamma=\left(\prod_{k=1}^{l+1} \sigma_{\max}(W_k)\right)^2$ leads to a feasible point for \eqref{eq:LipSDP}, directly recovering the norm product bound \eqref{eq:productbound}.  In addition, the derivations in  \cite{xueclipse} show that the choice of $\{\Lambda_k\}_{k=1}^{l+1}$ and $\gamma$ in ECLipsE-Fast also provides a feasible point for LipSDP. 
By comparing the expressions of $\{\Lambda_k\}$, one can see that the choice of $\{\Lambda_k\}$ for the norm product bound is decoupled in the sense that $\Lambda_k$ does not depend on $\Lambda_m$ for $m\neq k$. In contrast, 
ECLipsE-Fast exploits the structure of LipSDP to introduce the dependence of $\Lambda_k$ on $\Lambda_m$ (for $m<k$), leading to less conservative Lipschitz bounds.

The above discussion motivates us to ask the following question:
 Can we obtain other closed-form expressions of the feasible points of \eqref{eq:LipSDP} that will naturally lead to new closed-form Lipschitz bounds for \eqref{eq:NN}?
To give an affirmative answer to this question, we first obtain the following useful result.

\vspace{0.1in}

\begin{theorem}\label{thm:main}
Define $M_1=I_{n_0}$. Choose $\Lambda_1$ to be any invertible diagonal matrix satisfying $2\Lambda_1-\Lambda_1 W_1 W_1^\tp \Lambda_1\succ 0$.  Next, define $M_2=2\Lambda_1-\Lambda_1 W_1 M_1^{-1} W_1^\tp \Lambda_1$ and choose $\Lambda_2$ to be  any diagonal matrix satisfying $2\Lambda_2-\Lambda_2 W_2 M_2^{-1} W_2^\tp \Lambda_2\succ 0$. For $3\le k \le l$, continue this recursion by defining $M_k=2\Lambda_{k-1}-\Lambda_{k-1} W_{k-1} M_{k-1}^{-1} W_{k-1}^\tp \Lambda_{k-1}$ and choose $\Lambda_k$ to be  any diagonal matrix satisfying $2\Lambda_k-\Lambda_k W_k M_k^{-1} W_k^\tp \Lambda_k\succ 0$. Then LipSDP \eqref{eq:LipSDP} is feasible with the resultant choice of $\{\Lambda_k\}_{k=1}^l$ and $\gamma=\sigma_{\max}\left(W_{l+1}(2\Lambda_l-\Lambda_l W_l M_l^{-1} W_l^\tp \Lambda_l)^{-1}W_{l+1}^\tp\right)$.
\end{theorem}
\begin{proof}
The proof is similar to the proof of Theorem~3 in \cite{xueclipse}. The only difference is that we need to apply the Schur complement lemma for positive semidefinite matrices here, while Theorem 3 in \cite{xueclipse} relies on the simpler version of Schur complement lemma for positive definite matrices. Specifically, one can just recursively apply the Schur complement lemma for positive semidefinite matrices to \eqref{eq:LipSDP} for $(l+1)$ times, and that naturally leads to the conclusion.
\end{proof}

\vspace{0.1in}

Therefore, if we know how to find diagonal $\Lambda_k$ satisfying $2\Lambda_k-\Lambda_k W_k M_k^{-1} W_k\Lambda_k\succ 0$ for given $W_k$ and $M_k$ ($1\le k \le l)$, then we immediately have a recursive formula to compute the Lipschitz upper bound for the original neural network to be analyzed. 
To recover ECLipsE-Fast using Theorem \ref{thm:main}, we can simply  choose $\Lambda_k=\frac{1}{\sigma_{\max}(W_k M_k^{-1} W_k^\tp)}I$ and show
\begin{align*}
&2\Lambda_k-\Lambda_k W_k M_k^{-1} W_k^\tp \Lambda_k\\
=&\frac{2}{\sigma_{\max}(W_k M_k^{-1} W_k^\tp)}I-\frac{1}{\sigma_{\max}^2(W_k M_k^{-1} W_k^\tp)}W_k M_k^{-1} W_k^\tp\\
\succeq & \frac{2}{\sigma_{\max}(W_k M_k^{-1} W_k^\tp)}I-\frac{1}{\sigma_{\max}(W_k M_k^{-1} W_k^\tp)}I\\
=&\frac{1}{\sigma_{\max}(W_k M_k^{-1} W_k^\tp)}I\succ 0.
\end{align*}
By Theorem \ref{thm:main}, this directly recovers ECLipsE-Fast. 
Interestingly, there are many different ways to specify $\Lambda_k$ to ensure $2\Lambda_k-\Lambda_k W_k M_k^{-1} W_k\Lambda_k\succ 0$, which will be discussed next.

\subsection{A New Set of Closed-form Lipschitz Bounds for NNs}
\label{sec:LipBounds}

Now we leverage Theorem \ref{thm:main} to provide various new Lipschitz bounds for \eqref{eq:NN}. Later in Section \ref{sec:improve}, we will discuss how to use these bounds to improve ECLipsE-Fast. 

Suppose we will only use nonsingular $\Lambda_k$. Then we can rewrite the condition
$2\Lambda_k-\Lambda_k W_k M_k^{-1} W_k^\tp \Lambda_k\succ 0$ as
\begin{align}\label{eq:key}
\Lambda_k^{-1}\succ \frac{1}{2}W_k M_k^{-1} W_k^\tp
\end{align}
where $\Lambda_k^{-1}$ is known to be diagonal and positive definite.
Now we only need to think about how to choose a diagonal matrix such that it is greater than the matrix $\frac{1}{2}W_k M_k^{-1} W_k^\tp$ in the definite sense. We can immediately derive the following new Lipschitz bounds.

\vspace{0.1in}

\noindent\textbf{ECLipsE-SN:} For every $k$, we can choose $\Lambda_k$ as
\begin{align}\label{eq:feasible1}
\Lambda_k=\frac{c_k}{\sigma_{\max}(W_k M_k^{-1} W_k^\tp)}I,\,\, \forall c_k\in (0,2)
\end{align}
With this choice, it is straightforward to verify
$\Lambda_k^{-1}=\frac{1}{c_k}\sigma_{\max}(W_k M_k^{-1} W_k^\tp) I\succ \frac{1}{2}\sigma_{\max}(W_k M_k^{-1} W_k^\tp) I\succeq  \frac{1}{2} W_k M_k^{-1} W_k$, and hence \eqref{eq:key} holds. 
Therefore, by Theorem \ref{thm:main}, the recursive bound based on \eqref{eq:feasible1} and \eqref{eq:ECLipsE} gives a valid Lipschitz bound for \eqref{eq:NN}. We termed this bound as ECLipsE-SN since it is proved based on the spectral norm bound $W_k M_k^{-1} W_k^\tp \preceq \sigma_{\max}(W_k M_k^{-1} W_k^\tp) I$. 
One can view ECLipsE-Fast as a special case of ECLipsE-SN with $c_k=1$ for all $k$. In  Section \ref{sec:improve}, we will demonstrate the choice of $c_k=1$ introduces unnecessary conservatism.

\vspace{0.1in}

\noindent\textbf{ECLipsE-GC}: Next we use the Gershgorin circle theorem to derive ECLipsE-GC, which offers another Lipschitz bound. For simplicity, we define $\Gamma_k=W_k M_k^{-1} W_k^\tp$. By Gershgorin circle theorem, as long as the $i$-th diagonal entry of $\Lambda_k$ satisfies $\Lambda_k^{-1}(i,i)>\frac{1}{2}\sum_{j} \vert \Gamma_k(i,j)\vert$, then $(\Lambda_k^{-1}-\Gamma_k)$ is a diagonally dominant matrix, and \eqref{eq:key} holds as desired.  
Therefore, we can choose $\Lambda_k$ as a diagonal matrix
 whose $(i,i)$-entry is set up as follows. If $\sum_{j} \vert \Gamma_k(i,j)\vert>0$, we have
\begin{align}\label{eq:feasible2}
\Lambda_k(i,i)=\frac{c_k}{\sum_{j} \vert \Gamma_k(i,j)\vert},\,\,\,\forall c_k\in (0,2)
\end{align}
If $\sum_{j} \vert \Gamma_k(i,j)\vert=0$, we set
\begin{align}\label{eq:feasible3}
\Lambda_k(i,i)=\tilde{d},\,\,\,\forall \tilde{d}\in (0,\infty)
\end{align}
Combining \eqref{eq:feasible2}, \eqref{eq:feasible3}, and \eqref{eq:ECLipsE}, we directly obtain another Lipschitz bound. Interestingly, this bound replaces the largest singular value computation with the summation operation in \eqref{eq:feasible2}, providing extra computation advantages for extremely large-scale problems. 

\vspace{0.1in}

\noindent\textbf{ECLipsE-GCS}: We can use the well-known diagonal scaling variant of the Gershgorin circle theorem to further generalize ECLipsE-GC. Recall $\Gamma_k=W_k M_k^{-1} W_k^\tp$. Given arbitrary non-negative $\{q_i\}$, as long as $\Lambda_k^{-1}(i,i)>\frac{1}{2p_i}\sum_j p_j\vert \Gamma_k(i,j)\vert$, then \eqref{eq:key} holds. This is an immediate consequence of \cite[Corollary 6.1.6]{horn2012matrix}. 
In other words, we can choose \( \tilde{\Gamma}_k:= Q^{-1} \Gamma_k Q \) with arbitrary positive definite diagonal matrix $Q$, and then set the diagonal matrix $\Lambda_k$ as\footnote{Obviously the $(i,i)$-th entry of $Q$ is $q_i$. Notice \eqref{eq:feasible_GCS1} requires $\sum_{j} \vert \tilde{\Gamma}_k(i,j)\vert>0$. If $\sum_{j} \vert \tilde{\Gamma}_k(i,j)\vert=0$, we set
$\Lambda_k(i,i)=\tilde{d},\,\,\,\forall \tilde{d}\in (0,\infty)$.}
\begin{align}\label{eq:feasible_GCS1}
\Lambda_k(i,i)=\frac{c_k}{\sum_{j} \vert \bar{\Gamma}_k(i,j)\vert},\,\,\,\forall c_k\in (0,2)
\end{align}
Combining \eqref{eq:feasible_GCS1} and \eqref{eq:ECLipsE}
immediately leads to another Lipschitz bound. A natural choice for $\{q_i\}$ is $q_i=\Gamma_k(i,i)$ for $\Gamma_k(i,i)>0$ and $q_i=\epsilon>0$ with $\epsilon$ being small for $\Gamma_k(i,i)=0$. We can view ECLipsE-GCS as a scaling version of ECLipsE-GC. 
  Similar to ECLipsE-GC, the ECLipsE-GCS also enjoys similar computational advantage compared to ECLipse-Fast as it avoids the computation of spectral norm at each iteration.

\noindent\textbf{ECLipsE-Shift}: Given an arbitrary diagonal matrix $T_k$, then \eqref{eq:key} holds if and only if the following shifted matrix inequality holds
\begin{align}\label{eq:key2}
    \Lambda_k^{-1}-T_k\succ \frac{1}{2}\Gamma_k-T_k
\end{align}
A sufficient condition ensuring \eqref{eq:key2} to hold is given by
\[
\Lambda_k^{-1}-T_k \succ c_k \sigma_{max}(0.5\Gamma_k-T_k)I, \quad \forall c_k \in  (1,\infty)
\]
This immediately leads to another feasible point choice  
by specifying the diagonal entries of $\Lambda_k$ as
\begin{equation}\label{eq:feasible_Off1}
    \Lambda_k(i,i)=\frac{1}{T_k(i,i)+c_k \sigma_{max}(0.5 \Gamma_k-T_k)},\quad \forall c_k>1
\end{equation}
A natural choice for the diagonal matrix \(T_k\) is to set the $(i,i)$-th entry of $T_k$ as the $(i,i)$-th entry of $\frac{1}{2}\Gamma_k$.  Combining \eqref{eq:feasible_Off1} and \eqref{eq:ECLipsE} leads to ECLipsE-Shift, which leverages the shifting trick \eqref{eq:key2} to derive another Lipschitz bound.

\begin{remark} Based on the above derivations, it is clear that we can leverage various feasible points for \eqref{eq:key} to obtain Lipschitz bounds in different forms.
As a matter of fact, one can build upon our derivations to obtain even more choices of Lipschitz bounds. For example,
since \eqref{eq:key} is linear in $\Lambda_k^{-1}$, we can interpolate the choice of $\Lambda_k^{-1}$ in the above bounds to easily derive new Lipschitz bounds. However, 
    it remains unclear how to compare the feasible points used in our various Lipschitz bounds without calculating the numerical values. In Section \ref{sec:improve}, we will discuss how to combine these bounds to reduce the conservatism in ECLipsE-Fast.
\end{remark}

\subsection{Reducing Conservatism in ECLipsE-Fast}
\label{sec:improve}

Each Lipschitz bound above corresponds to a particular feasible point of LipSDP \eqref{eq:LipSDP}. However, it is challenging to determine a priori which feasible point will yield the tightest bound, as in general these points are not directly comparable. The effectiveness of a particular feasible solution depends on the specific structure of the neural network and the properties of the input domain. In this section, we will first illustrate the potential conservatism of ECLipsE-Fast by explaining why choosing $c_k=1$ $\forall k$ may not be optimal for ECLipsE-SN, and then discuss how to leverage all the bounds in Section \ref{sec:LipBounds} to reduce such unnecessary conservatism.

Recall that ECLipsE-Fast is a special case of ECLipsE-SN with $c_k=1$ $\forall k$. In \cite{xueclipse}, some geometric explanations are provided to justify the development of ECLipsE-Fast. Specifically, the main intuition is based on \cite[Proposition 4]{xueclipse} which states that the choice of $\Lambda_k$ in \eqref{eq:feasible1} with $c_k=1$ minimizes $\sigma_{\max}((2\Lambda_k-\Lambda_k W_k M_k^{-1} W_k^\tp \Lambda_k)^{-1})$. However, an important issue is that there is no direct relationship between $\sigma_{\max}((2\Lambda_k-\Lambda_k W_k M_k^{-1} W_k^\tp \Lambda_k)^{-1})$ and the final resultant Lipschitz bound $\sigma_{\max}(W_{l+1}(2\Lambda_l-\Lambda_l W_l M_l^{-1} W_l^\tp \Lambda_l)^{-1}W_{l+1}^\tp)$. 
To see this, we can first set $k=l$.
When $M_l$ is given, the optimal choice of $\Lambda_l$ should be the one that minimizes  $\sigma_{\max}(W_{l+1}(2\Lambda_l-\Lambda_l W_l M_l^{-1} W_l^\tp \Lambda_l)^{-1}W_{l+1}^\tp)$ and obviously depends on $W_{l+1}$. In contrast, the choice of $\Lambda_l$ from ECLipsE-Fast does not depend on $W_{l+1}$, and hence is not optimal in general. Similarly, if we set $k=l-1$ and let $M_{l-1}$ be given, then the minimization of $\sigma_{\max}(W_{l+1}(2\Lambda_l-\Lambda_l W_l M_l^{-1} W_l^\tp \Lambda_l)^{-1}W_{l+1}^\tp)$ over $\{\Lambda_{l-1},\Lambda_l\}$ cannot be easily decoupled in a way that one can choose $\Lambda_{l-1}$ and $\Lambda_l$ separately. Continuing this argument, we can see that finding the optimal choice of $\{c_k\}$ for ECLipsE-SN is a subtle problem involving complex coupling between different $k$. Choosing $c_k=1$ $\forall k$ as ECLipsE-Fast may not be optimal.

In practice, it is quite difficult to determine a priori what choice of $\{c_k\}$ for ECLipsE-SN will give the best numerical Lipschitz bound for a given NN.
As a matter of fact, we do not even know how to select the Lipschitz bounds from Section \ref{sec:LipBounds} for a particular NN to be analyzed. Now we discuss some heuristics that can be used to reduce the conservatism in ECLipsE-Fast when computing Lipschitz bounds for any given neural networks. Fortunately, whenever we fix the values of $\{c_k\}$, all the bounds in Section \ref{sec:LipBounds} can be efficiently calculated with similar computational cost as ECLipsE-Fast. Therefore, we can choose $c_k=c$ $\forall k$ and then do grid search or bisection on the scalar hyperparameter $c$ to figure out the best choice for each bound. With these choices of $c$ for different bounds, we can calculate all the bounds in Section \ref{sec:LipBounds} efficiently and then report the best numerical value. This simple heuristic method immediately reduce the conservatism in ECLipsE-Fast while maintaining almost the same level of scalability. In the next section, we will perform numerical experiments to show that the best choice among the Lipschitz bounds in Section \ref{sec:LipBounds} is problem-dependent, and our proposed heuristic method reduces the conservatism in ECLipsE-Fast efficiently.

\section{Numerical Experiments}
\label{sec:num}

In this section, we provide numerical experimental results to support our claim that our new Lipschitz bounds can be combined to reduce the conservatism in ECLipsE-Fast. 
When implementing our bounds including ECLipsE-SN, ECLipsE-GC, ECLipsE-GCS, and ECLipsE-Shift, we set $c_k=c$ $\forall k$  and treat  $c$ as a hyperparameter. To push for the best numerical values of the Lipschitz bounds, one typically needs to fine-tune the value of $c$ via search or bisection. Due to the scalar nature, the tuning process of $c$ remains simple and computationally efficient.

Since our main point is that our derived bounds can reduce the conservatism in ECLipsE-Fast, it makes sense for us to adopt the same experimental setups from \cite{xueclipse}. 
Specifically, the experiments in \cite{xueclipse} consider networks trained on the MNIST dataset or generated randomly with varying depth/width. 
In this section,
our dervied bounds including ECLipsE-SN, ECLipsE-GC, ECLipsE-GCS, and ECLipsE-Shift are first compared with ECLipsE-Fast on the exact MNIST networks provided in the supplementary material of \cite{xueclipse}.  For randomly generated networks used in \cite{xueclipse}, the exact weight information has not been provided online. Instead, a code for generating random networks was originally provided in the supplementary material of \cite{xueclipse} for the purpose of reproduction. We directly download this code to generate various random networks with varying depth/width and then compare our bounds with ECLipsE-Fast on those resultant NNs. Our comparison is based on two key metrics: the numerical values of the resultant Lipschitz bounds and the computational time required for their evaluation. Notice that similar to \cite{xueclipse}, we systematically vary the number of neurons per layer in fully connected networks as well as the overall depth of the network to make our experimental study comprehensive and convincing.

\subsection{Lipschitz Bounds for MNIST Networks}

First, we directly download the MNIST networks provided by the online supplementary material of \cite{xueclipse}.
These are three-layer neural networks with \(\{100, 200, 300, 400\}\) neurons in each layer, trained on the MNIST dataset. Table~\ref{tab:mnist_comp} presents a comparison of the 
Lipschitz bounds computed using ECLipsE-Fast and the improved methods derived in Section \ref{sec:LipBounds}. 
Regarding the specific hyperparameter choices, the results in Table~\ref{tab:mnist_comp} are obtained using \( c = 1.3 \) for ECLipsE-SN, \( c = 1.99 \) for ECLipsE-GC and ECLipsE-GCS, and \( c = 1.7 \) for ECLipsE-Shift. The results in Table \ref{tab:mnist_comp} clearly demonstrate that for the MNIST networks used in \cite{xueclipse},  ECLipsE-SN with $c=1.3$, ECLipsE-GCS with $c=1.99$, and ECLipsE-Shift with $c=1.7$ can immediately improve the original results from ECLipsE-Fast. Interestingly, ECLipsE-Shift provides the most improvement in this case.
Our comparison also supports our claim  that the choice of \( c= 1 \) in ECLipsE-Fast is not always optimal for ECLipsE-SN. In this case, setting \( c = 1.3 \) in ECLipsE-SN can trivially reduce the conservatism.

Since the networks used in this analysis are relatively shallow, the computational time across different methods remains approximately the same, and a detailed discussion is omitted here.

\begin{table}
    \centering
    \vspace{0.1in}
    \begin{tabular}{c p{1cm}<{\raggedleft} p{1cm}<{\raggedleft} p{1cm}<{\raggedleft} p{1cm}<{\raggedleft} p{1cm}<{\raggedleft}}
        Neurons   & ECLipsE-Fast & ECLipsE-SN & ECLipsE-GC & ECLipsE-GCS & ECLipsE-Shift \\
        \hline
         & & & & & \\
        100   &  18.79 &   17.40  &  18.23 &  17.71  & \textbf{17.32 }     \\
        200   &  19.66 &   19.29  &  20.58 &  19.28  & \textbf{19.04 }     \\
        300   &  19.50 &   18.75  &  20.82 &  19.20  & \textbf{18.44 }    \\
        400   &  19.92 &   19.04  &  21.38 &  19.92  & \textbf{18.92 }     \\
        \hline
    \end{tabular}\caption{Comparing Lipschitz bounds on MNIST networks from \cite{xueclipse}. The best bounds are highlighted in bold.}
    \label{tab:mnist_comp}
\end{table}

\begin{table}
    \centering
   \begin{tabular}{c p{1cm}<{\raggedleft} p{1cm}<{\raggedleft} p{1cm}<{\raggedleft} p{1cm}<{\raggedleft} p{1cm}<{\raggedleft}}
    Depth   & ECLipsE-Fast & ECLipsE-SN & ECLipsE-GC & ECLipsE-GCS & ECLipsE-Shift\\
    \hline
    20   &  0.31  &  0.31  &  \textbf{0.30}  &  0.33  &  0.31  \\
    30   &  2.20  &  2.20  &  \textbf{2.11}  &  2.40  &  2.20  \\
    50   & 39.53  & 39.53  &  \textbf{37.43}  & 44.50  & 39.48  \\ 
    75   &  5.63  &  5.63  &  \textbf{5.21}  &  6.62  &  5.62  \\ 
    100  & 74.57  & 74.57  &  \textbf{67.64}  & 91.33  & 74.40  \\
    \hline
    \end{tabular}
    \caption{Lipschitz bounds with varying depth. \\The best bounds are highlighted in bold.}
    \label{tab:DNN_lyrs}
\end{table}

\begin{table}[t!]
    \centering
    \begin{tabular}{c p{1cm} p{1cm} p{1cm} p{1cm} p{1cm}}  
        Width   & ECLipse-Fast & ECLipse-SN & ECLipse-GC & ECLipse-GCS & ECLipse-Shift\\
        \hline
         & & & & & \\
    80    &   0.04     &    0.04      &  \textbf{0.036}     &       0.05    &       0.04 \\   
    100   &  74.57     &   74.57      &  \textbf{67.64}     &      91.33    &      74.45 \\   
    120   &  15.30     &   15.30      &  \textbf{14.01}     &      18.35    &      15.28 \\   
    140   &  27.84     &   27.84      &  \textbf{25.72}     &      32.64    &      27.80 \\   
    160   &   0.08     &    0.08      &  \textbf{ 0.07}     &       0.09    &       0.08 \\
        \hline
    \end{tabular}
    \caption{Lipschitz bounds with Varying Width.\\ The best bounds are highlighted in bold. }
    \label{tab:DNN_neurons}
\end{table}

\subsection{Randomly Generated Neural Networks}
For deep NNs, it is crucial that the Lipschitz bound computation scales efficiently with both the depth and width of the network. To assess the scalability of our bounds, we conducted a comprehensive comparison study similar to \cite{xueclipse} by varying the depth (number of layers) and width (number of neurons per layer) of NNs, and the results are reported in Table \ref{tab:DNN_lyrs} and Table \ref{tab:DNN_neurons}, respectively. In line with \cite{xueclipse}, we used their code available online to randomly generate weight matrices. 
Specifically, for Table~\ref{tab:DNN_lyrs}, we generated neural networks with \(\{20,30, 50, 75, 100\}\) layers, where each layer consists of 100 neurons. For Table~\ref{tab:DNN_neurons}, we
 randomly sample weight matrices for 100-layer neural networks with \(\{80, 100, 120, 140, 160\}\) neurons per layer. 

For both Table~\ref{tab:DNN_lyrs} and Table \ref{tab:DNN_neurons}, 
the networks are getting larger and hence we 
adopt a simpler choice of the hyperparameter $c$ as follows: \( c = 1.0 \) for ECLipsE-SN, \( c = 1.0 \) for ECLipsE-GC and ECLipsE-GCS, and \( c = 2.0 \) for ECLipsE-Shift. 
Notably, in this setting, if we actually change the value of $c$ to $1.3$,  the resultant numerical value of the Lipschitz bounds actually gets larger.
This actually demonstrates that the optimal choice of $c$ for ECLipsE-SN is highly problem-dependent, and one has to rely on search or bisection heuristics to explore the best choice of the $c$ value.

 In Table~\ref{tab:DNN_lyrs}, both ECLipsE-GC and ECLipsE-Shift outperform ECLipsE-Fast consistently, with ECLipsE-GC exhibiting the most significant improvement as the network depth increases from 20 layers to 100 layers.
 In Table \ref{tab:DNN_neurons}, ECLipsE-GC again consistently provides the best numerical values for the resultant Lipschitz bounds.
 We emphasize that the results in Table~\ref{tab:DNN_lyrs} and  Table \ref{tab:DNN_neurons} do not mean that one should always use ECLipsE-GC in the deep network setting. However, these results do demonstrate the potential benefits of combining ECLipsE-SN, ECLipsE-GC, ECLipsE-GCS, and ECLipsE-Shift to reduce the conservatism in ECLipsE-Fast.

For deep neural networks, the computational time required to estimate the bounds is another crucial factor. Figure~\ref{fig:DNN_time} presents a comparison of the computational time required by the proposed techniques as the number of neurons per layer increases from 80 to 160. Since ECLipsE-GC and ECLipsE-GCS replace the largest singular value computation with sum-based operators, they outperform ECLipsE-Fast in terms of computational efficiency. This demonstrates that these methods not only produce tighter Lipschitz bounds but also significantly reduce computational overhead, making them more suitable for large-scale neural networks.

\begin{figure}
    \centering
    \includegraphics[width=1\linewidth]{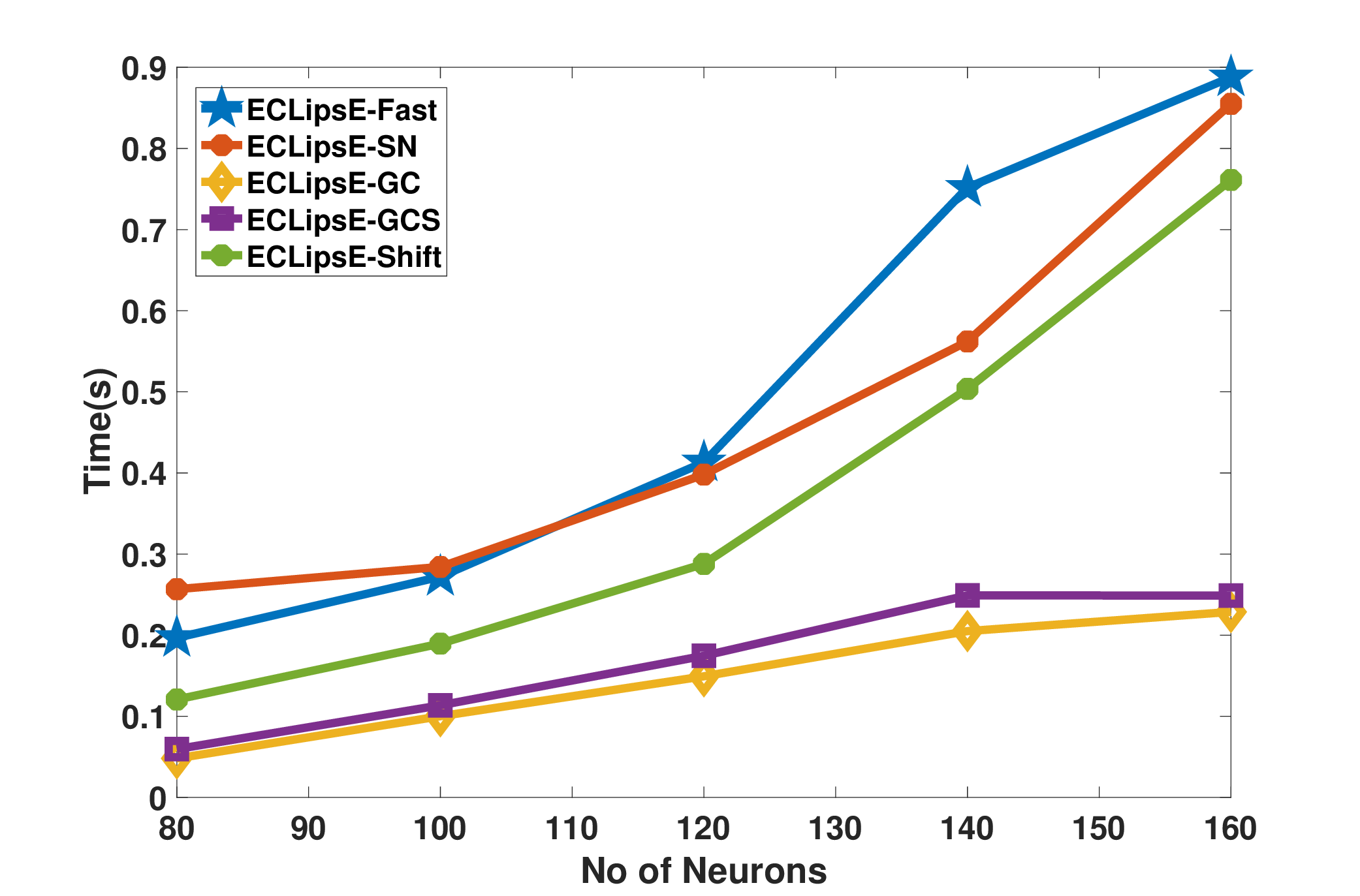}
    \caption{Computation time}
    \label{fig:DNN_time}
\end{figure}

\section{Conclusion}
\label{sec:conclud}

In this paper, we develop a family of new scalable Lipschitz bounds that do not require solving SDPs and can reduce the  
conservatism in the previous SOTA non-SDP bound ECLipsE-Fast. We provide a principled approach that leverages analytical solutions of a particular matrix inequality to streamline the developments of scalable Lipschitz bounds for deep networks.  Our numerical results support our claim that our derived Lipschitz bounds are less conservative than ECLipsE-Fast while maintaining the same level of scalability.

\bibliographystyle{IEEEtran}
\bibliography{main}

\end{document}